\documentclass[runningheads]{llncs}
\usepackage{graphicx}
\usepackage{times}
\usepackage[linewidth=1pt]{mdframed}
\usepackage{xhfill}
\usepackage{soul}
\usepackage{url}
\usepackage[utf8]{inputenc}
\usepackage[small]{caption}
\usepackage{csquotes}
\usepackage{graphicx}
\usepackage{bm}
\usepackage{amssymb}
\usepackage{booktabs}
\usepackage{cprotect}
\usepackage{fancybox}
\usepackage{sansmath}
\usepackage{mathtools}
\usepackage{pgf}
\usepackage{tikz}
\usepackage{pgfplots}

\usepackage{amsfonts} 
\usepackage{amsmath} 
\usepackage{enumitem} 

\usepackage{float} 

\usepackage{textcomp} 
\usepackage{siunitx} 
\usepackage{algorithm,algpseudocode}
\usepackage{pgfplotstable}
\definecolor{dgreen}{rgb}{0.157,0.654,0.292}

\algtext*{EndIf}
\algtext*{EndFor}

\algdef{SE}[DOWHILE]{Do}{doWhile}{\algorithmicdo}[1]{\algorithmicwhile\ #1}%

\makeatletter
\renewcommand{\Function}[2]{%
	\csname ALG@cmd@\ALG@L @Function\endcsname{#1}{#2}%
	\def\jayden@currentfunction{#1}%
}
\newcommand{\funclabel}[1]{%
	\@bsphack
	\protected@write\@auxout{}{%
		\string\newlabel{#1}{{\jayden@currentfunction}{\thepage}}%
	}%
	\@esphack
}
\makeatother

\algdef{SE}[VARIABLES]{Variables}{EndVariables}
{\algorithmicvariables}
{\algorithmicend\ \algorithmicvariables}
\algnewcommand{\algorithmicvariables}{\textbf{global variables}}

\algrenewcommand\algorithmiccomment[1]{\hfill #1}

\algtext*{EndWhile}
\algtext*{EndIf}

\newcounter{ctTODO}
\newcommand\filltoend{\leavevmode{\unskip
		\leaders\hrule height.5ex depth\dimexpr-.5ex+0.4pt\hfill\hbox{}%
		\parfillskip=0pt\endgraf}}

\begin{document}
	\title{Towards Utilitarian Combinatorial Assignment with Deep Neural Networks and Heuristic Algorithms\thanks{This work was partially supported by the Wallenberg AI, Autonomous Systems and Software Program (WASP) funded by the Knut and Alice Wallenberg Foundation, and by grants from the National Graduate School in Computer Science (CUGS), Sweden, Excellence Center at Linköping-Lund for Information Technology (ELLIIT), TAILOR funded by EU Horizon 2020 research and innovation programme (GA 952215), and Knut and Alice Wallenberg Foundation (KAW 2019.0350).}}
	\titlerunning{Towards Utilitarian Combinatorial Assignment}
	%
	
	\author{Fredrik Pr\"{a}ntare \and
		Mattias Tiger \and David Bergstr\"{o}m \and \\ Herman Appelgren \and Fredrik Heintz}

	
	%
	\authorrunning{F. Pr\"{a}ntare et al.}
	%
	\institute{Link\"{o}ping University\\581 83 Link\"{o}ping, Sweden\\
		\email{\{firstname.lastname\}@liu.se}}
	
	\maketitle              
	%
	
	
	
	
	\begin{abstract}
		This paper presents preliminary work on using deep neural networks to guide general-purpose heuristic algorithms for performing utilitarian combinatorial assignment. In more detail, we use deep learning in an attempt to produce heuristics that can be used together with e.g., search algorithms to generate feasible solutions of higher quality more quickly. Our results indicate that our approach could be a promising future method for constructing such heuristics.
		

	\end{abstract}
	\section{Introduction}
	A major problem in computer science is that of designing cost-effective, scalable \emph{assignment} algorithms that seek to find a \emph{maximum weight matching} between the elements of sets.  We consider a highly challenging problem of this type---namely \emph{utilitarian combinatorial assignment} (UCA), in which indivisible elements (e.g., items) have to be distributed in bundles (i.e., partitioned) among a set of other elements (e.g., agents) to maximize a notion of aggregated value. This is a central problem in both artificial intelligence, operations research, and algorithmic game theory; with applications in optimal task allocation \cite{prantare2017simultaneous}, winner determination for auctions \cite{sandholm2002winner}, and team formation \cite{prantare2020anytime}.
	
	However, UCA is computationally hard. The state-of-the-art can only compute solutions to problems with severely limited input sizes---and due to Sandholm \cite{sandholm2002algorithm}, we expect that no polynomial-time approximation algorithm exists that can find a feasible solution with a provably good worst-case ratio. With this in mind, it is interesting to investigate if and how low-complexity algorithms can generate feasible solutions of high-enough quality for problems with large-scale inputs and limited computation budgets. In this paper, we present preliminary theoretical and experimental foundations for using function approximation algorithms (e.g., neural networks) together with heuristic algorithms to solve UCA problems. 

	\section{Related Work}
	
	The only UCA algorithm in the literature is an optimal branch-and-bound algorithm \cite{prantare2018anytime,prantare2020anytime}. Although this algorithm greatly outperforms industry-grade solvers like CPLEX in difficult benchmarks, it can only solve fairly small problems. 
	
	Furthermore, a plethora of heuristic algorithms \cite{sen2000searching,keinanen2009simulated,di2010coalition,yeh2016solving,farinelli2017hierarchical} have been developed for the closely related \emph{characteristic function game coalition structure generation} (CSG) problem, in which we seek to find an (unordered) utilitarian partitioning of a set of agents. However, due to the CSG problem's ``unordered nature'', all of these methods are unsuitable for UCA unless e.g., they are redesigned from the ground up.
	
	Apart from this, there has been considerable work in developing algorithms for the \emph{winner determination problem} (WDP) \cite{andersson2000integer,sandholm2002winner,sandholm2002algorithm}---in which the goal is to assign a subset of the elements to alternatives called \emph{bidders} in a way that maximizes an auctioneer's profit. WDP differs from UCA in that the value function is not given in an exhaustive manner, but instead as a list (often constrained in size) of explicit ``bids'' that reveal how much the bidders value different bundles of items. 
	
	Moreover, heuristic search with a learned heuristic function has in recent years achieved super-human performance in playing many difficult games. A key problem in solving games with massive state spaces is to have a sufficiently good approximation (heuristic) of the value of a sub-tree in any given state. Recent progress within the deep learning field with \emph{multi-layered (deep) neural networks} has made learning such an approximation possible in a number of settings. \cite{lecun2015deep,silver2016mastering} 
	
	Using a previously learned heuristic function in heuristic search is an approach of integrating machine learning and combinatorial optimization (CO), and it is categorized as \emph{machine learning alongside optimization algorithms} \cite{bengio2018machine}. Another category is \emph{end-to-end learning}, in which machine learning is leveraged to learn a function that outputs solutions to CO problems directly. While, the end-to-end approach has been applied to graph-based problems such as the \emph{traveling salesman problem} \cite{khalil2017learning} and the \emph{propositional satisfiability problem} \cite{selsam2018learning}, the learned heuristic-based approach remain both a dominating and more fruitful approach \cite{silver2018general,schrittwieser2019mastering}. 
	
	\section{Problem Description}
	The UCA problem that we investigate is defined as the following optimization problem:
	
	\noindent \filltoend
	
	\noindent\textbf{Input:}  A set of elements $A=\{a_1, ..., a_n\}$, a set of alternatives ${T=\{ t_1, ..., t_m \}}$, and a function \mbox{$\bm{v} : 2^A \times T \mapsto \mathbb{R}$} that maps a value to every possible pairing of a bundle $C\subseteq A$ to an alternative $t \in T$.
	
	\medskip 
	
	\noindent\textbf{Output:} A \textit{combinatorial assignment}  (Definition \ref{combinatorial-assignment})  $\langle C_1, ..., C_m \rangle$ over $A$ that maximizes $\sum_{i=1}^{m} \bm{v}(C_i, t_i)$.

	\noindent\filltoend
	
	\begin{definition}\label{combinatorial-assignment}
		$S = \langle C_1,...,C_m \rangle$ is a \emph{combinatorial assignment} over $A$ if $C_i \cap C_j = \emptyset$ for all $i \neq j$, and \mbox{$\bigcup_{i=1}^{m} C_{i} = A$}.
	\end{definition}
	
	Note that there are applications for which it is realistic (or even preferred) to have the value function given in this type of fashion. Examples of this include the strategy game \textit{Europa Universalis 4}, where it is given by the game's programmers to enforce a certain behaviour from its game-playing agents \cite{prantare2020anytime}. Other such examples include when it can be defined concisely but e.g., not given as an explicit list due to the problem's size, such as in winner determination for combinatorial auctions \cite{sandholm2002winner}, or when the value function is a result of machine (e.g., reinforcement) learning.
	
	
	Moreover, we use $\bm{V}(S) = \sum_{i=1}^{m} \bm{v}(C_i, t_i)$ to denote the value of a \textit{partial assignment} (Definition \ref{partial-assignment}) \mbox{$S = \langle C_1,...,C_m \rangle$}, and define $||S|| = \sum_{i=1}^m|C_i|$. The terms \emph{solution} and \emph{combinatorial assignment} are used interchangeably, and we often omit ``over $A$'' for brevity. We also use $\Pi_A$ for the set of all combinatorial assignments over $A$, and define $\Pi_{A}^m = \{ S \in \Pi_A : |S| = m \}$. We say that a solution $S^*$ is \textit{optimal} if and only if $\bm{V}(S^*) = \max_{S \in \Pi^m_A}\bm{V}(S)$.
	
	\begin{definition}\label{partial-assignment}
		If $S$ is a combinatorial assignment over some $A' \subseteq A$, we say that $S$ is a \emph{partial assignment} over $A$. 
	\end{definition}
	
	(Note that we are intentionally using a non-strict inclusion in Definition \ref{partial-assignment} for practical reasons. Consequently, a combinatorial assignment is also a partial assignment over the same agent set.)
	
	%
	
	
	Now, to formally express our approach to UCA, first let $\langle a'_1,...,a'_n \rangle$ be any permutation of $A$, and define the following recurrence:
	\begin{equation}\label{eq:rec} 
	\bm{V}^*(S) = 
	\begin{cases*} 
	\bm{V}(S ) & if $||S|| = n$ \\ 
	\max_{S' \in \bm{\Delta} (S, a'_{||S|| + 1})}\bm{V}^*(S') & otherwise
	\end{cases*}
	\end{equation}
	where $\bm{\Delta}(\langle C_1,...,C_m \rangle ,a) = \{ \langle C_1 \cup \{a\},...,C_m \rangle,...,\langle C_1,...,C_m \cup \{a\} \rangle \}$, and $S$ is a combinatorial assignment over $\{a'_1,...,a'_{||S||}\}$. As a consequence of Theorem \ref{th:uca}, UCA boils down to computing recurrence (\ref{eq:rec}). Against this background, in this paper, we investigate approximating $\bm{V}^*$, in a dynamic programming fashion, using neural networks together with heuristic methods with the goal to find better solutions quicker.

	\begin{theorem}\label{th:uca}
		$\bm{V}^*(S) = \max_{S \in \Pi^m_A}\bm{V}(S)$ if $S = \langle C_1,...,C_m \rangle$ is a partial assignment over $A$ with $C_i = \emptyset$ for $i = 1,...,m$.
	\end{theorem}
	\begin{proof}
		This result follows in a straightforward fashion by induction. \qed
	\end{proof}
	
	%
	
	%
	


	
	\section{Heuristic Function Model and Training}
	
	We approximate (\ref{eq:rec}) with a fully connected \textit{deep neural network} (DNN) $f_\theta({S})$ with parameters $\theta$, where ${S}$ is a partial assignment. Our DNN has three hidden layers using \emph{ReLU} \cite{lecun2015deep} activation functions. Each hidden layer has width $mn + 1$. The input is a $m\times n$ \emph{binary assignment-matrix} representation of ${S}$, and a scalar with the partial assignment's value $\bm{V}({S})$. See Fig. \ref{fig:NN-architecture} for a visual depiction of our architecture.

	\begin{figure}[h]
		\centering
		\includegraphics[width=0.7\textwidth]{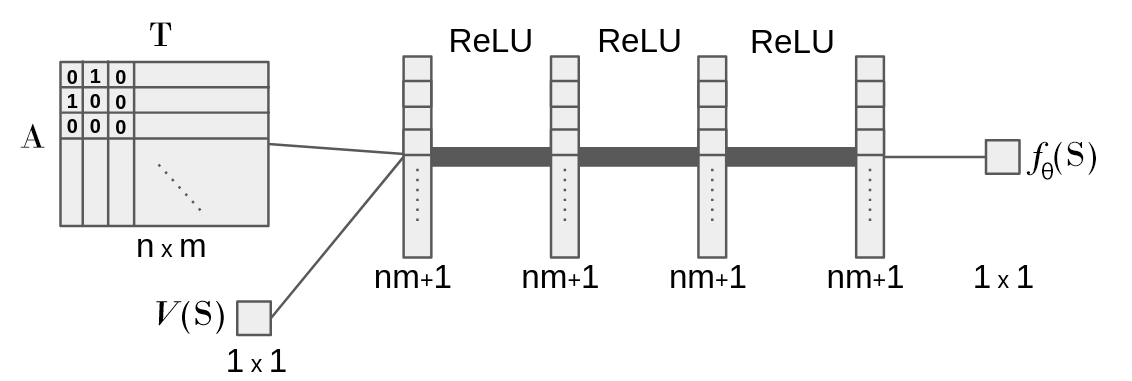}
		\caption{Our multi-layered neural network architecture.}
		\label{fig:NN-architecture}
	\end{figure}
	Our training procedure incorporates generating a data set $\mathcal{D}$ that consists of pairs $\langle {S}, \bm{V}^*({S}) \rangle$, with randomized partial assignments ${S} \in \Pi_{A_i}^m$, where $A_i\subset A$ is a uniformly drawn subset from $A$ with $|A_i|=i$, for $i = n-1,\dots,n-\kappa$, where $\kappa \in \{1,...,n\}$ is a hyperparameter. In our experiments, $\mathcal{D}$ consists of exactly $10^4$ such pairs for every $i$. Note that it is only tractable to compute $\bm{V}^*$ if $\kappa$ is kept small, since in such cases we only have to search a tree with depth $\kappa$ and branching factor $m$ to compute the real value of $\bm{V}^*$. For this reason, we used $\kappa \leq 10$ in our benchmarks. $\theta$ is optimized over the training data using stochastic optimization to minimize:
	\begin{align}
	\mathbb{E}_{\langle {S}, \bm{V}^*({S}) \rangle\sim\mathcal{D}}\big[\bm{V}^*({S}) - f_\theta({S})\big].
	\label{eq:loss}
	\end{align}
	The data set is split $90\%/10\%$ into a training set and a test set. The stochastic optimizer Adam \cite{kingma2014adam} is used for minimizing (\ref{eq:loss}) over the training set. The hyperparameters \textit{learning rate} and \textit{mini-batch size} are optimized using grid search over the test set.  In our subsequent experiments, the same $\bm{V}$ is used as the one used for generating $\mathcal{D}$ by storing the value function's values.  
 
	
	\section{Experiments}
	We use the problem set distributions \emph{NPD} (\ref{eq:npd-definition}) and \emph{TRAP} (\ref{eq:trap-definition}) for generating difficult problem instances for evaluating our method. NPD is one of the more difficult standardized problem instances for optimal solvers \cite{prantare2020anytime}, and it is defined as follows:
	\begin{align}
	\bm{v}(C,t) &\sim \mathcal{N}(\mu,\sigma^2),\label{eq:npd-definition}
	\end{align}
	
	\noindent for $C \subseteq A$ and $t\in T$. TRAP is introduced by us in this paper, and it is defined with:
	\begin{align}
	\bm{v}(C, t) &\sim \mathcal{N}(\bm{\tau}(C),\sigma^2),\label{eq:trap-definition}
	\end{align}
	for all $C \subseteq A$ and $t \in T$, where:
	\begin{align*}
	\bm{\tau}(C) &= \delta \begin{cases}
	|C| - |C|^2, &  0\leq |C| < \tau\\
	|C| - |C|^2 + |C|^{(2 + \epsilon)}, &  \tau\leq |C|
	\end{cases}
	\end{align*}
	for all $C \subseteq A$. $\bm{\tau}$ is defined to make it difficult for general-purpose greedy algorithms that work on an element-to-element basis to find good solutions by providing a ``trap''. This is because when $\epsilon > 0$, they may get stuck in (potentially arbitrarily bad) local optima, since for TRAP, the value $\bm{V}(S)$ of a partial assignment $S$ typically provides little information about $\bm{V}^*(S)$. In contrast, for NPD, $\bm{V}(S)$ can often be a relatively accurate estimation for $\bm{V}^*(S)$. It is thus interesting to deduce whether our learned heuristic can overcome this problem, and consequently outperform greedy approximations.
	
	\begin{figure}[h]
		\centering
		\includegraphics[width=0.46\textwidth]{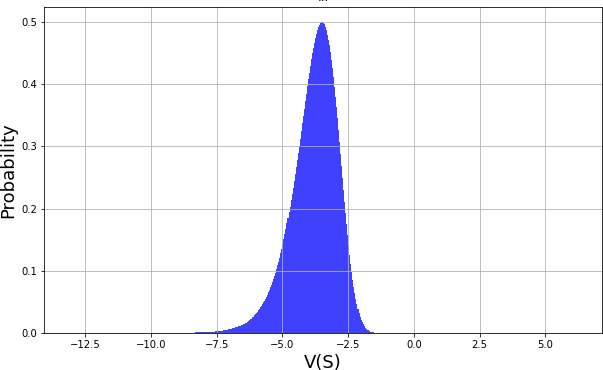}~\hspace{1.0em}~
		\includegraphics[width=0.46\textwidth]{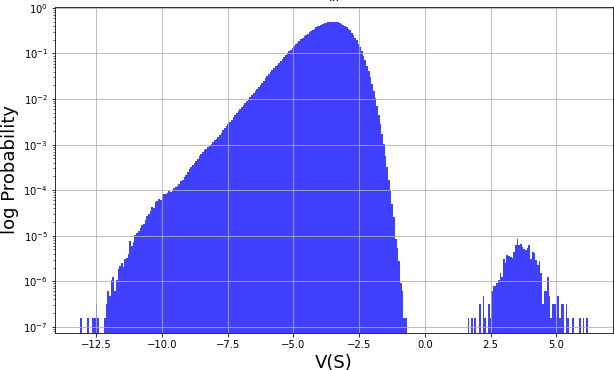}
		\caption{Empirical estimation of $\bm{P}\big(\bm{V}(S)\big)$ for TRAP from $10^8$ samples.}
		\label{fig:empirical-distribution-trap}
	\end{figure}

	We used  $n=20$, $m=10$, $\mu = 1$, $\sigma = 0.1$, $\delta = 0.1$, $\tau = n/2$ and $\epsilon = 0.1$ in our experiments. $n$ and $m$ are chosen to be small enough for exact methods to be tractable.

	To give an idea of how difficult it is to find good solutions for TRAP, we plot an empirical estimation of it in Fig. \ref{fig:empirical-distribution-trap}, generated using $10^8$ draws with (\ref{eq:trap-definition}). The probability of drawing a combinatorial assignment $S \in \Pi_{A}^m$ at random with a value larger than zero, i.e., $\bm{P}\big(\bm{V}(S) > 0\big)$, is approximately $7.43\times 10^{-6}$ (only 743 samples found). This was computed using Monte Carlo integration with $10^8$ samples. 


	\subsection{Training Evaluation}
	For NPD, Fig. \ref{fig:NN-learning} shows that our neural network generalizes from the 1-5 unassigned elements case to 6-10 unassigned elements with only a slight degradation in prediction error variance (figures to the left). We also see that the predictions are slightly worse for predicting higher assignment values than lower ones, but that the performance is fairly evenly balanced otherwise.
	
	Similar figures for TRAP are also shown in Fig. \ref{fig:NN-learning}. Here, the prediction error variance is very high around 5-7 unassigned elements. In the right-most figure, we see that the neural network has problems predicting assignment values close to TRAP's ``jump'' (i.e., $|C| = \tau$). However, outside of value ranges close to the jump, the prediction performance is decent, if not with as high precision as for NPD. Note that TRAP is trained for 1-10 unassigned elements, so no generalization is evaluated in this experiment.
	
		\begin{figure}[H]
		\centering
		\includegraphics[width=0.98\textwidth]{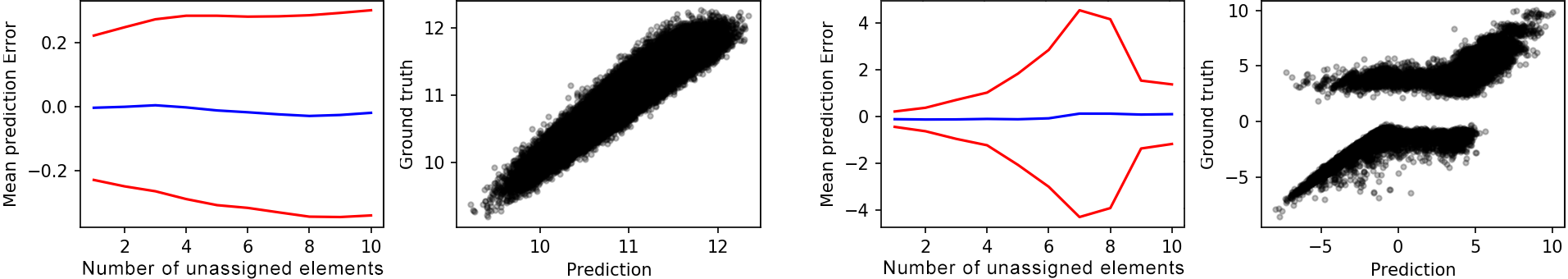}
		\caption{Neural network results for {NPD} (left) and {TRAP} (right). The left figure of each pair shows the mean prediction error and 2 std., as a function of the number of unassigned elements of the partial assignment. The right shows the predicted value compared to the true value.}
		\label{fig:NN-learning}
	\end{figure}
	
	Despite the seemingly large prediction error variance we find that the neural network has a narrower prediction distribution than an uninformed guess. More so for NPD than TRAP, but even a slightly better prediction than random is helpful as a heuristic. This is especially true for TRAP-like distributions, since for them, we previously had no better alternative. Moreover, the prediction errors' distributions are seemingly unbiased.


		\subsection{Benchmarks}
	
		The result of each experiment in the following benchmarks was produced by computing the average of the resulting values from 5 generated problem sets per experiment. In these benchmarks, the goal is to give an indication how well our neural networks perform compared to more naïve approaches for estimating the optimal assignment's value (and thus their suitability when integrated in a heuristic). These estimation methods are coupled with a standard greedy algorithm to draw samples from the search space. We use the following baseline estimations: 1) \textit{current-value} estimation, which uses the partial assignment's value (so that each evaluation becomes a greedily found local optimum); and 2) a random approach, which is a worst-case baseline based on a random estimation (so that each evaluation is a uniformly drawn sample from the search space). The best solution drawn over a number of samples is then stored and plotted in Fig. \ref{benchmarks}. The 95\% confidence interval is also plotted.  The results show that our neural network is able to overcome some problems element-to-element-based heuristics may face with TRAP. For NPD, it performs almost identical to the current-value greedy approach.

	\section{Conclusions}
	We have made the first theoretical and experimental foundations for using deep neural networks together with heuristic algorithms to solve utilitarian combinatorial assignment problems. Albeit much remains to be explored and tested (including generalization, difficulty, what is learned, etcetera), our preliminary results and simple function approximator show that using neural networks together with heuristic algorithms could be a promising future method for finding high-quality combinatorial assignments. 
	
		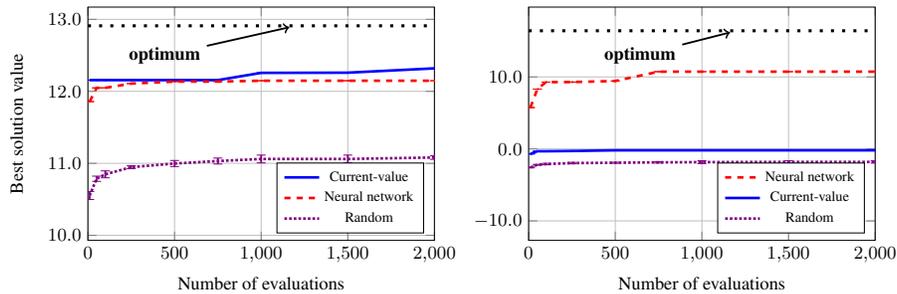
\begin{figure}[H]
		\centering
		\begin{tikzpicture}[scale=0.77]
			\begin{axis}[
				error bars/y dir      = both,
				error bars/y explicit = true,
				error bars/error bar style={solid,mark size=2.2pt,line width=0.9pt},
				enlargelimits=true,
				enlarge x limits=false,
				grid=both,
				height=0.46\textwidth,
				width=0.62\textwidth,
				ymin=10.2, 
				ymax=12.9,
				xtick={0,500,1000,1500,2000},
				xmin=0,
				xmax=2000,
				legend pos= south east,
				y tick label style={
					/pgf/number format/.cd,
					fixed,
					fixed zerofill,
					precision=1,
					/tikz/.cd
				},
				xlabel=Number of evaluations,
				ylabel={Best solution value}
				]      
				
				\node[anchor=west] (source) at (axis cs:200,12.5){\textbf{optimum}};
				\node (destination) at (axis cs:1200,12.91){};
				\draw[->,line width = 0.9pt](source)--(destination);
				
				\addplot[mark= ,blue, solid,line width = 1.3pt] coordinates  
				{ 
					(10,12.1557) +- (0,0)
					(50,12.1557) +- (0,0)
					(100,12.1557) +- (0,0)
					(250,12.1557) +- (0,0)
					(500,12.1557) +- (0,0)
					(750,12.1557) +- (0,0)
					(1000,12.2565) +- (0,0)
					(1500,12.259) +- (0,0)
					(2000,12.3191) +- (0,0)
				};

				\addplot[mark= ,red, dashed,line width = 1.3pt] coordinates  
				{ 
					(10,11.8575) +- (0,0)
					(50,12.048) +- (0,0)
					(100,12.048) +- (0,0)
					(250,12.1073) +- (0,0)
					(500,12.1356) +- (0,0)
					(750,12.1356) +- (0,0)
					(1000,12.1471) +- (0,0)
					(1500,12.1471) +- (0,0)
					(2000,12.1471) +- (0,0)
				};
				
				\addplot[mark= ,violet, densely dotted,line width = 1.3pt] coordinates  
				{ 
					(10,10.5531) +- (0.0552173,0.0552173)
					(50,10.7858) +- (0.0361769,0.0361769)
					(100,10.8471) +- (0.048015,0.048015)
					(250,10.9467) +- (0.0215386,0.0215386)
					(500,10.9972) +- (0.0422321,0.0422321)
					(750,11.0316) +- (0.0413594,0.0413594)
					(1000,11.0604) +- (0.0540559,0.0540559)
					(1500,11.0604) +- (0.0540559,0.0540559)
					(2000,11.0808) +- (0.0258338,0.0258338)
				};
				
				\addplot[mark= ,black, loosely dotted,line width = 1.5pt] coordinates  
				{ 
					(0.0,12.91)
					(2000,12.91)
				};
				
				\addlegendentry{\scriptsize{Current-value}}
				\addlegendentry{\scriptsize{Neural network}}
				\addlegendentry{\scriptsize{Random}}
				
			\end{axis}
		\end{tikzpicture}
		\begin{tikzpicture}[scale=0.77]
			\begin{axis}[
				error bars/y dir      = both,
				error bars/y explicit = true,
				error bars/error bar style={solid,mark size=2.2pt,line width=0.9pt},
				enlargelimits=true,
				enlarge x limits=false,
				grid=both,
				height=0.46\textwidth,
				width=0.62\textwidth,
				ymin=-10.0, 
				ymax=17.0,
				xtick={0,500,1000,1500,2000},
				xmin=0,
				xmax=2000,
				legend pos= south east,
				y tick label style={
					/pgf/number format/.cd,
					fixed,
					fixed zerofill,
					precision=1,
					/tikz/.cd
				},
				xlabel=Number of evaluations,
				]      
				
				\node[anchor=west] (source) at (axis cs:400,13.00){\textbf{optimum}};
				\node (destination) at (axis cs:1200,16.41){};
				\draw[->,line width = 0.9pt](source)--(destination);
				
				\addplot[mark= ,red, dashed,line width = 1.3pt] coordinates  
				{ 
					(10,5.74089) +- (0,0)
					(50,8.31005) +- (0,0)
					(100,9.28347) +- (0,0)
					(250,9.28347) +- (0,0)
					(500,9.43765) +- (0,0)
					(750,10.7306) +- (0,0)
					(1000,10.7306) +- (0,0)
					(1500,10.7306) +- (0,0)
					(2000,10.7306) +- (0,0)
				};
				
				\addplot[mark= ,blue, solid,line width = 1.3pt] coordinates  
				{ 
					(10,-0.699285) +- (0,0)
					(50,-0.322589) +- (0,0)
					(100,-0.322589) +- (0,0)
					(250,-0.287955) +- (0,0)
					(500,-0.194121) +- (0,0)
					(750,-0.194121) +- (0,0)
					(1000,-0.194121) +- (0,0)
					(1500,-0.194121) +- (0,0)
					(2000,-0.194121) +- (0,0)
				};

				\addplot[mark= ,violet, densely dotted,line width = 1.3pt] coordinates  
				{ 
					(10,-2.57577) +- (0.0579972,0.0579972)
					(50,-2.20486) +- (0.0620165,0.0620165)
					(100,-2.10302) +- (0.0983733,0.0983733)
					(250,-1.96686) +- (0.0973956,0.0973956)
					(500,-1.92167) +- (0.0746969,0.0746969)
					(750,-1.86598) +- (0.0983546,0.0983546)
					(1000,-1.8378) +- (0.197373,0.197373)
					(1500,-1.80886) +- (0.157265,0.157265)
					(2000,-1.80886) +- (0.157265,0.157265)
				};

				\addplot[mark= ,black, loosely dotted,line width = 1.5pt] coordinates  
				{ 
					(0.0,16.411)
					(2000,16.411)
				};
				
				\addlegendentry{\scriptsize{Neural network}}
				\addlegendentry{\scriptsize{Current-value}}
				\addlegendentry{\scriptsize{Random}}
				
			\end{axis}
		\end{tikzpicture}
		\caption{\small{The best solution values obtained by the different heuristics when using a greedy algorithm for {{NPD}} (left) and {{TRAP}} (right) problem sets with 20 elements and 10 alternatives.\label{benchmarks}}}
	\end{figure}

	\bibliographystyle{splncs04}
	\bibliography{ecai20}
	
\end{document}